\documentclass[10pt]{article}
\usepackage{arxiv}

\usepackage{amsmath,amsfonts,bm}

\def\eqref#1{equation~\ref{#1}}

\def\1{\bm{1}}

\DeclareMathAlphabet{\mathsfit}{\encodingdefault}{\sfdefault}{m}{sl}
\SetMathAlphabet{\mathsfit}{bold}{\encodingdefault}{\sfdefault}{bx}{n}

\usepackage{natbib}
\usepackage{wrapfig}
\usepackage{hyperref}
\usepackage{url}
\usepackage{enumitem}
\usepackage{algorithm}
\usepackage[noend]{algpseudocode}\algrenewcommand\algorithmicrequire{\textbf{Input:}}

\usepackage{amsmath}
\usepackage{amssymb}
\usepackage{mathtools}
\usepackage{amsthm}

\usepackage{thm-restate}
\usepackage{nicefrac}
\usepackage{tikz}
\usetikzlibrary{calc,fit,positioning,arrows}
\usetikzlibrary{arrows.meta}

\theoremstyle{plain}
\newtheorem{theorem}{Theorem}[section]

\theoremstyle{definition}
\newtheorem{definition}[theorem]{Definition}

\theoremstyle{remark}

\newcommand{\abs}[1]{\left\vert #1\right\vert}

\newcommand{\N}{\mathbb{N}}

\DeclarePairedDelimiterX{\infdivx}[2]{(}{)}{%
  #1\;\delimsize\|\;#2%
}

\usepackage{xcolor}

\title{Theoretical Barriers in Bellman-Based Reinforcement Learning}

\author{Brieuc Pinon, Raphaël Jungers, Jean-Charles Delvenne\\
Department of Mathematical Engineering\\
UCLouvain\\
Belgium \\
\texttt{\{brieuc.pinon,raphael.jungers,jean-charles.delvenne\}@uclouvain.be}
}

\begin{document}

\maketitle

\begin{abstract}
    Reinforcement Learning algorithms designed for high-dimensional spaces often enforce the Bellman equation on a sampled subset of states, relying on generalization to propagate knowledge across the state space. In this paper, we identify and formalize a fundamental limitation of this common approach. Specifically, we construct counterexample problems with a simple structure that this approach fails to exploit. Our findings reveal that such algorithms can neglect critical information about the problems, leading to inefficiencies. Furthermore, we extend this negative result to another approach from the literature: Hindsight Experience Replay learning state-to-state reachability.
\end{abstract}

\section{Introduction}
    One of the main goals of Artificial Intelligence is to devise universal algorithmic ideas to solve problems as efficiently as possible. One typical application is Automated Theorem Proving (ATP), which serves as a running example application here. In ATP, solving a problem involves finding a proof for a theorem from a set of axioms within the constraints of a logic system. Since this must be done under constrained computational resources, algorithms must be designed to find a solution quickly.

    A promising direction to design efficient general-purpose algorithms is to blend search and learning, where learning is leveraged to accelerate the search process. These algorithms iteratively attempt to construct solutions, leveraging feedback from previous attempts to learn to guide the next constructions.

    Reinforcement Learning (RL) implements this paradigm to optimize objectives over sequences of decisions, where each decision-making point is associated with a state. This structure allows Dynamic Programming to be applied in the form of the Bellman equation to learn a value function. RL algorithms typically sample sequences of states guided by the value function while simultaneously improving this value function by enforcing the Bellman equation on the sampled states. This approach forms the foundation of many RL algorithms, with various adaptations \citep{sutton2018reinforcement}.

    This paper examines the theoretical limitations of this RL approach to enhance our understanding and support the development of more effective algorithms. 

    Intuitively, the results of this paper formalize the incapacity of these algorithms to efficiently ``learn from failure''. For instance, in ATP, the probability of proving a theorem with randomly sampled logical rules is often low, this leads to sparse rewards and might hinder the capacity of an algorithm to learn. While proving an incapacity to solve a problem with sparse rewards without any other information is trivial ---since any algorithm would resort to exhaustive search--- we analyze the non-trivial case: when the algorithm can leverage a priori information to help his search and learn from its failures.
    
    To study this limitation, we analyze a concrete RL algorithm that applies the Bellman equation with Bayesian Learning. The algorithm starts with a set of hypothesis value functions and iteratively refines this prior through the Bellman equation. Bayesian Learning provides us with an idealized framework for incorporating prior knowledge while refining it with experience.
    
    Under natural assumptions on the initial prior, we prove a key limitation. Despite the strengths of the RL approach and the learning procedure, the algorithm struggles with certain counterexample problems with an ``easy'' design. In contrast, we show that a classical symbolic algorithm is not subject to such a limitation. This suggests that the highlighted limitation is not a fatality but rather a shortcoming due to algorithmic choices.

    The counterexamples are based on the following principle: combining a set of problems into one aggregated problem that should be (with an appropriate strategy) nearly as easy to solve as the original sub-problems. However, by carefully designing the aggregation, we can obscure the information of each individual sub-problem, forcing the Bellman equation-based RL algorithm to tackle all original sub-problems simultaneously rather than breaking them down into independent, simpler tasks.
    
    A potential remedy to this limitation is expanding the feedback leveraged by the algorithm. One such approach is Hindsight Experience Replay (HER) \citep{andrychowicz2017hindsight}. HER builds upon the Bellman equation to learn \emph{universal value functions} \citep{sutton2011horde,schaul2015universal}. Assuming that the objective is to reach some goal-state, a standard value function only estimates the value of a state in terms of its ability to reach this goal. In contrast, a universal value function predicts the reachability of any state from any other state. HER utilizes the states encountered during the learning process to learn this universal value function, leveraging richer feedback than a simple binary outcome of whether the goal was reached or not. With this consideration, several works in the literature apply HER in ATP \citep{aygun2022proving,trinh2024solving,poesia2024learning}. 

    However, this new feedback also misses critical information to solve our counterexamples efficiently. We straightforwardly extend our analysis to HER with state-to-state universal value functions, proving that the same limitation applies.
    
    \paragraph{Related work} Close to our work, \citet{sun2019model} constructed a family of problems that model-free RL methods struggle to solve efficiently compared to a model-based method with access to a predetermined goal. Subsequent research expanded these findings by demonstrating a similar limitation broadened to some model-based methods, moreover, it did not require a priori knowledge of a goal \citep{pinon2025limitation}. Both of these works highlighted limitations in Bellman equation-based methods by embedding critical information within unknown dynamics.

    \paragraph{Contributions} Instead of relying on uncertainty in the dynamics, our analysis builds on the computational limitations of evaluating (universal) value functions. This method enables our contributions:
    \begin{itemize}
        \item Demonstrating a limitation for an algorithm based on the Bellman equation on a family of problems with an a priori known common dynamics across the problems. This extends the applications in which an inefficiency for a Bellman equation-based method can be expected, particularly for Automated Theorem Proving.
        \item Formalizing a limitation for an algorithm relying on the Bellman equation with universal value functions to learn state-to-state reachability, such as in Hindsight Experience Replay.
    \end{itemize}

    \paragraph{Outline} In the preliminaries, we provide convenient notations and definitions. In Section \ref{sec:Bell-value}, we present an algorithm implementing the Bellman equation applied to learn a value function, and then derive a limitation for that algorithm. In Section \ref{sec:Bell-universal}, we extend that limitation to an HER algorithm. Finally, in Section \ref{sec:disc}, we discuss our results and conclude.

\section{Preliminaries}
    We note $[n]=\{1,\ldots, n\}$ the set of the $n$ first natural numbers. For a vector $x\in X^n$, with some set $X$ and $n\in\N$, we note $x_{\leq i}(/x_{<i})$ the vector restricted to the first $i(/i-1)$th coordinates. An index list $I$ over $n\in\N$ is a sequence of numbers in $[n]$. For $I$ an index list over $n$, $x_I$ is the vector composed of the values of $x$ at the coordinates in $I$. For $x$ and $y$ two vectors, we note $[x,y]$ their concatenation.

    Our counterexamples are based upon the \textit{Boolean satisfiability} problem, we define here a classical form of this problem. Given a vector of $n\in\N$ Boolean variables $x_i$ ($i\in [n]$): a \textit{literal} is one of these variables $x_i$ or its negation $\neg x_i$; a \textit{clause} is a set of literals joined by disjunctions $\lor$. Finally, a \textit{conjunctive normal form satisfiability} (CNF-SAT) instance is a set of clauses over the $n$ variables and a solution for the instance is a vector $x\in\{\mathrm{False},\mathrm{True}\}^n$ such that all the clauses evaluate to $\mathrm{True}$ under the interpretation given by $x$. We denote this evaluation function checking a binary vector $x$ against a CNF-SAT instance $p$ $\mathrm{Check}(x;p)$.

    We take two liberties with respect to this formalism. One, we make no distinction between the Boolean and binary space by confounding $0(/1)$ and $\mathrm{False}(/\mathrm{True})$ respectively. Two, in all the formal parts of the paper, we keep the name ``CNF-SAT instance'' (following the computational complexity literature) but we also denote them by the generic term of ``problem'' in other parts of the paper.

\section{Limitation on the Bellman Equation with Value Functions}\label{sec:Bell-value}
    
    We identify a limitation of Bellman equation-based algorithms when applied to solving CNF-SAT instances. Specifically, we show that these algorithms can fail to exploit structure in aggregated instances. We demonstrate this through Algorithm \ref{alg:BE_search}, which implements the Bellman equation to learn a value function and guide the search towards solutions for a CNF-SAT instance. Theorem \ref{thm:BE} establishes that the algorithm does not effectively decompose some counterexample aggregated instances, resulting in an exponential runtime with respect to the number of instances.
    
    \subsection{The Algorithm}
        To solve a CNF-SAT instance using RL, we formulate the problem as a sequence of binary decisions. An instance with $n$ binary variables is modeled as a sequence of $n$ binary decisions that form a candidate solution. For each decision, the state is defined as the SAT instance and the first $i$ fixed variables, while the next action involves assigning a value ($0$ or $1$) to the $i+1$th variable. The RL algorithm seeks to optimize these decisions under the objective of maximizing $1$ for a solution and $0$ otherwise.

        To guide the generation toward a solution, our RL algorithm leverages value functions learned by enforcing the Bellman equation. We now define these value functions, their optimality, and the Bellman equation.

        \begin{definition}
            A \emph{value function} $v(x_{\leq i};p)$ maps a CNF-SAT instance $p$ over $n$ variables and a partial assignment $x_{\leq i}\in\{0,1\}^i$ with $i\in [n]$ to $\{0,1\}$.
        \end{definition}
    
        \begin{definition}
            A value function $v$ is \emph{optimal} if for any CNF-SAT instance $p$ with $n$ variables and any partial assignment $x_{\leq i}$, $v(x_{\leq i};p)=1$ iff there exists an extension $y\in\{0,1\}^{n-i}$ such that $\mathrm{Check}([x_{\leq i},y];p)$ is True.
        \end{definition}

        In our case, the Bellman equation applied on some value function $v$ at some partial assignment $x_{\leq i}$ is:
        \begin{equation} \label{eq:Bellman}
            v(x_{\leq i};p)=\max\{v([x_{\leq i},0];p),\,v([x_{\leq i},1];p)\},
        \end{equation}
        where $i\in [n-1]$ and $p$ is the CNF-SAT instance to solve. Given that our setup involves deterministic transitions, this expression does not use expectations. Also, since the objective is binary, we use value functions with a binary output.

        The Bellman equation must also enforce consistency with $\mathrm{Check}$ when evaluating a complete assignment. This additional constraint reads $v(x;p)=\mathrm{Check}(x;p)$ where $x$ is any complete assignment.
        
        \begin{algorithm}[t]
        \caption{Learning value functions with the Bellman equation for CNF-SAT.\\
        The algorithm iteratively samples candidate solutions guided by a hypothesis set of value functions. In parallel, the Bellman equation is iteratively enforced on the set. The sampling procedure follows $\pi^{V^t}(p)$ defined in Definition \ref{def:pi}.}
        \label{alg:BE_search}
        \textbf{Inputs:} $V^0:$ an initial set of value functions; $p:$ the CNF-SAT instance to solve with $n$ variables
        \begin{algorithmic}
            \For{$t\gets 0,1,\ldots$}
                \State $x^t\sim \pi^{V^t}(p)$
                \If{$\mathrm{Check}(x^t;\,p)$}
                    \State \textbf{output} $x^t$
                \EndIf
                \State $V^{t+1}\gets \{v\in V^t|\,v(x^t;p)=0\}$
                \For{$i\gets 1,\ldots,n-1$}
                    \State $V^{t+1}\gets\{v\in V^{t+1}|\, v(x_{\leq i}^t;p)=\max\{v([x_{\leq i}^t,0];p),\,v([x_{\leq i}^t,1];p)\}\}$
                \EndFor
            \EndFor
        \end{algorithmic}
        \end{algorithm}

        We now define Algorithm \ref{alg:BE_search}, it iteratively generates candidate solutions guided by values' estimates while simultaneously improving these estimates through the Bellman equation.
        
        The algorithm starts with an initial set of hypotheses for the value functions, denoted as $V^0$, and iteratively samples sequences of decisions to construct candidate solutions, guided by the current set $V^t$. At each step, the Bellman equation is applied to eliminate inconsistent value functions from $V^t$, continuing this process until a solution is found.
    
        The initial hypothesis set $V^0$ encodes a prior over possible value functions. This prior can incorporate knowledge about solving SAT instances, value functions excluded from $V^0$ reduce the set of possibilities and potentially accelerate the search. This knowledge may originate from some human-coded prior or from training on other SAT instances.
    
        Likewise, the set $V^0$ also encodes uncertainties. When $V^0$ contains value functions providing different estimates, this reflects uncertainty and can force the algorithm to explore multiple possibilities. These possibilities are then progressively discarded with the Bellman equation.

        Algorithm \ref{alg:BE_search} thus implements a form of Bayesian Learning where a set of hypotheses is iteratively refined based on new evidence. Traditional Bayesian Learning employs a probabilistic distribution over the hypothesis space, but in this case, Equation \ref{eq:Bellman} provides a deterministic condition for consistency. This eliminates the need for a probabilistic framework, allowing us to work directly with the set of hypotheses. However, we note that the results presented here are transposable to the case where probabilistic distributions are used.

        The sampling of candidate solutions is guided by the current set of value functions $V^t$, and the sampling procedure is formalized as follows:
        \begin{definition}\label{def:pi}
            For a given set of value functions $V$, we define $\pi^V$ as a mapping from any CNF-SAT instance $p$ with $n$ variables to a probability distribution over the space of binary candidate solutions $x=[x_0,\ldots,x_n]\in\{0,1\}^n$.

            Let $V_{z,p}$ denote the number of value functions in $V$ that evaluate to $1$ for inputs $z\in\{0,1\}^{[n]}$ and $p$, i.e., $\abs{\{v\in V|\,v(z;p)=1\}}$.

            The probability distribution is defined incrementally as
            \begin{equation}\label{eq:pi}
                \pi^V(x_i=0;p,x_{<i})=\frac{V_{[x_{<i},0],p}}{V_{[x_{<i},0],p}+V_{[x_{<i},1],p}},
            \end{equation}
            with $\pi^V(x_i=1;p,x_{<i})=1-\pi^V(x_i=0;p,x_{<i})$. If the denominator in Equation \ref{eq:pi} equals zero, both actions are assigned equal probabilities.
        \end{definition}

        This procedure converts a set of hypothesis value functions $V$ into a sampling policy. The resulting samples are sequentially constructed candidate solutions, where each bit $x_i$ is appended based on an estimated probability of the existence of a solution.

    \subsection{Counterexamples}
        With our algorithm defined, we now present a negative result by constructing examples of hard instances. To do so, we introduce an operation that aggregates multiple CNF-SAT instances. This operation is formally defined as follows:
        \begin{definition}\label{def:aggregation}
            Let $p_1,\ldots,p_K$ be CNF-SAT instances over $n_1,\ldots,n_K$ variables, respectively. Additionaly, let $I_1,\ldots,I_K$ be index lists over $n=n_1+\ldots+n_K$ variables such that $\abs{I_1}=n_1,\ldots,\abs{I_K}=n_K$, and all the elements in these lists are distinct. An \emph{aggregation} of $p_1,\ldots,p_K$ with index lists $I_1,\ldots,I_K$ produces a new CNF-SAT instance $p$ over $n$ variables, where all the literals in the clauses are mapped according to the index lists. For example a clause $x_i\lor\neg x_j$ in $p_k$ becomes $x_{I_{k,i}}\lor\neg x_{I_{k,j}}$ in $p$.
        \end{definition}

        To support our result, we also introduce a monotonic property for value functions.
        \begin{definition}
            A value function $v$ is \emph{monotonic} if, for any CNF-SAT instance $p$ with $n$ variables, any $x\in\{0,1\}^n$, and any $i,j\in[n]$ with $i<j$, the function satisfies: $v(x_{\leq i};p)\geq v(x_{\leq j};p)$.
        \end{definition}
        This property implies that, from any partially constructed solution, each decision can only decrease the value function's output. It is a characteristic of any optimal value function.
        
        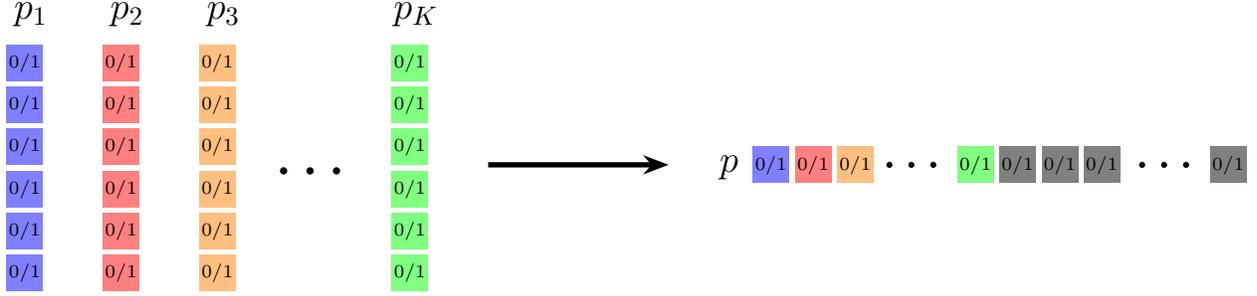
\begin{figure*}[ht]
            \centering
            \begin{tikzpicture}[scale=1.6,
        arrow/.style={->,line width=2pt,>=stealth,>={Stealth[scale=0.8]}}]
            \def\squareSize{0.3} 
            \def\squareGap{0.05} 
        
            \def\rectSize{5}   
        
            \foreach \y in {0,1,...,\rectSize} {
                \fill[blue!50] (0, \y * \squareSize + \y * \squareGap) 
                rectangle ++(\squareSize, \squareSize);
                \node at (0.15, \y * \squareSize + \y * \squareGap + 0.15) {\scriptsize $0/1$};
            }
            \node at (2 * \squareSize + 2 * \squareGap-0.5, 2.3) {\Large \(p_1\)};
        
            \begin{scope}[xshift=0.8cm] 
                \node at (2 * \squareSize + 2 * \squareGap -0.5, 2.3) {\Large \(p_2\)};
                \foreach \y in {0,1,...,\rectSize} {
                    \fill[red!50] (0, \y * \squareSize + \y * \squareGap) 
                    rectangle ++(\squareSize, \squareSize);
                    \node at (0.15, \y * \squareSize + \y * \squareGap + 0.15) {\scriptsize $0/1$};
                }
            \end{scope}
    
            \begin{scope}[xshift=1.6cm] 
                \node at (2 * \squareSize + 2 * \squareGap -0.5, 2.3) {\Large \(p_3\)};
                \foreach \y in {0,1,...,\rectSize} {
                    \fill[orange!50] (0, \y * \squareSize + \y * \squareGap) 
                    rectangle ++(\squareSize, \squareSize);
                    \node at (0.15, \y * \squareSize + \y * \squareGap + 0.15) {\scriptsize $0/1$};
                }
            \end{scope}
    
            \node at (2.56, 1.) {\Huge \dots};
        
            \begin{scope}[xshift=3.2cm] 
                \node at (2 * \squareSize + 2 * \squareGap-0.5, 2.3) {\Large \(p_K\)};
                \foreach \y in {0,1,...,\rectSize} {
                    \fill[green!50] (0, \y * \squareSize + \y * \squareGap) 
                    rectangle ++(\squareSize, \squareSize);
                    \node at (0.15, \y * \squareSize + \y * \squareGap + 0.15) {\scriptsize $0/1$};
                }
            \end{scope}
    
            \draw[arrow] (4., 1.05) -- (5.5, 1.05);
    
            \begin{scope}[xshift=6.2cm,yshift=0.9cm] 
                \node at (-0.2,0.13) {\Large $p$};
                \fill[blue!50] (0, 0) rectangle ++(\squareSize, \squareSize);
                \node at (0.15, 0.15) {\scriptsize $0/1$};
        
                \fill[red!50] (\squareSize + \squareGap, 0) rectangle ++(\squareSize, \squareSize);
                \node at (\squareSize + \squareGap+0.15, 0.15) {\scriptsize $0/1$};
                \fill[orange!50] (2*\squareSize + 2*\squareGap, 0) rectangle ++(\squareSize, \squareSize);
                \node at (2*\squareSize + 2*\squareGap+0.15, 0.15) {\scriptsize $0/1$};
        
                \node at (4 * \squareSize + 3 * \squareGap, 0.15) {\huge \dots};
        
                \fill[green!50] (5 * \squareSize + 4 * \squareGap, 0) rectangle ++(\squareSize, \squareSize);
                \node at (5*\squareSize + 4*\squareGap+0.15, 0.15) {\scriptsize $0/1$};
        
                \foreach \i in {0,1,2} {
                    \fill[gray] (6*\squareSize+\i*\squareSize + 5*\squareGap+\i*\squareGap, 0) rectangle ++(\squareSize, \squareSize);
                    \node at (6*\squareSize+\i*\squareSize + 5*\squareGap+\i*\squareGap +0.15, 0.15) {\scriptsize $0/1$};
                }
        
                \node at (10 * \squareSize + 9 * \squareGap, 0.15) {\huge \dots};
        
                \fill[gray] (11 * \squareSize + 10 * \squareGap, 0) rectangle ++(\squareSize, \squareSize);
                \node at (11*\squareSize + 10*\squareGap+0.15, 0.15) {\scriptsize $0/1$};
            \end{scope}
            \end{tikzpicture}
            \caption{Representation of a counterexample problem. Independent sub-problems $p_0,\ldots,p_K$ are aggregated into a single composite problem $p$, where the first variable of each sub-problem is mapped at the start of the new problem. Under appropriate assumptions, Theorem \ref{thm:BE} states that this construction forces a Bellman equation-based algorithm (Algorithm \ref{alg:BE_search}) to have an exponential runtime in the number of sub-problems $K$.}
            \label{fig:aggregation}
        \end{figure*}
    
        We now have all the necessary elements to state our result. Theorem \ref{thm:BE} demonstrates that, under certain assumptions, an aggregated problem is computationally intractable for Algorithm \ref{alg:BE_search}. The proof leverages the fact that, while the Bellman equation is sufficient to eventually identify an optimal value function and produce a solution, it does so inefficiently in this context.
        
        Importantly, the Bellman equation relies on the outputs of the value function to learn. When a solution attempt fails, the output $0$ of the value function (indicating failure) provides minimal feedback. This lack of informative feedback prevents the algorithm from understanding the reasons behind the failure and learning to anticipate similar failures. 
    
        Theorem \ref{thm:BE} constructs a problem with this difficulty by forcing Algorithm \ref{alg:BE_search} to make several hard independent binary decisions at the start. If any one of these decisions is incorrect, the attempt fails (value $0$) and the algorithm struggles to attribute the failure to any particular decision.

        A representation of the aggregation operation used in our counterexamples is provided in Figure \ref{fig:aggregation}.

        \begin{restatable}{theorem}{bellmanvaluethm}\label{thm:BE}
            Let $p$ be a CNF-SAT instance over $n$ variables, constructed by an aggregation of CNF-SAT instances $p_1,\ldots,p_K$ using index lists $I_1,\ldots,I_k,\ldots,I_K$, where the first element of each $I_k$ is $k$. 
    
            Let $V$ and $V_1,\ldots,V_K$ represent sets of value functions, and let $v^*$ denote an optimal value function.
            
            Assume the following:
            \begin{enumerate}
                \item The set $V$ factorizes into $V_1,\ldots,V_K$; that is, $v\in V$ iff there exist $v_1\in V_1,\ldots,v_K\in V_K$ such that for all $x\in\{0,1\}^n$ and $i\in[n]$, $v(x_{\leq i};p)=\prod\limits_{k\in[K]}v_{k}(x_{I_k\cap [i]};p)$.

                \item For all $k\in[K]$, either $v^*(0;p_k)=1$ or $v^*(1;p_k)=1$.

                \item For all $k\in[K]$ and $v\in V_k$, either $v(0;p_k)=1$ or $v(1;p_k)=1$.
                
                \item For all $k\in[K]$, if $v^*(0;p_k)=1$, then $\pi^{V_k}(x_0=0;p_k)\leq\pi^{V_k}(x_0=1;p_k)$; otherwise, $\pi^{V_k}(x_0=0;p_k)\geq\pi^{V_k}(x_0=1;p_k)$.
    
                \item Any $v\in V$ is monotonic.
            \end{enumerate}
    
            Under these assumptions, Algorithm \ref{alg:BE_search}, initialized with $V^0=V$ and $p=p$, runs for an expected time of at least $2^{K-1}$ steps.
        \end{restatable}

        The complete proof is deferred to Appendix \ref{sec:proofs}. Here, we outline the key ideas. 

        By assumption (2), only one possible assignment for the $K$th first variables leads to a solution. Due to assumptions (1) and (4), this assignment has a low prior probability under $V^0$.

        The algorithm tries to iteratively improve its prior through the Bellman equation. However, due to assumptions (1), (3), and (5), only a small portion of the hypothesis space is inconsistent with the equation on the generated states. Consequently, the prior is not substantially improved at each iteration.
        
        By quantifying these statements, it follows that the algorithm requires $2^{K-1}$ expected steps to solve the problem.

        \paragraph{Discussion of Assumptions in Theorem \ref{thm:BE}}
        
        Assumption (1) requires that the prior over value functions $V^0$ reflects the structure of the aggregated problem. Specifically, the set of value functions must be factorizable into sets of value functions corresponding to each individual sub-problem. This assumption is natural when sub-problems are independent, meaning that learning something about one sub-problem does not aid in solving the others.

        Assumption (2) posits that the first bit of each sub-problem is critical to solving that sub-problem. SAT instances exhibiting this property can be straightforwardly constructed. Assumption (3) asks that the value functions in the prior $V^0$ reflect that structure. Any optimal value function satisfies this condition.
    
        Assumption (4) asks for uncertainty in the prior $V^0$. With this condition, the prior $V^0$ does not provide sufficient guidance to decide optimally the first bit of each sub-problem. This can be the case if deciding the first bits of the SAT instances is difficult due to a computational barrier in evaluating value estimates. This assumption ensures that each sub-problem necessitates some degree of search to be solved.
    
        Assumption (5) stipulates that any value function in the prior is monotonic. Monotonicity is a reasonable property, as any optimal value function for SAT instances must be monotonic.
    
        \paragraph{Interpretation of Theorem \ref{thm:BE}}
        Theorem \ref{thm:BE} highlights a key limitation of relying on the Bellman equation to refine a prior over value functions: when faced with aggregated problems that are not directly solved by the prior, the resulting problem can become computationally intractable. For the algorithm, the expected running time grows at least exponentially with the number of aggregated sub-problems. 
        
        This issue lies in the inability of the Bellman equation to effectively decompose the aggregated problem, thereby hindering the learning of an optimal value function that could guide the search.

        The relevance of our result, and its interpretation as a limitation, are based on the intuition that a ``good'' algorithm should not struggle from the aggregation of several sub-problems. Ideally, solving an aggregated problem should only incur a modest computational overhead compared to solving each problem independently.
    
        In contrast, SAT solvers based on resolution techniques \citep{silva1996grasp,biere2009handbook,knuth2015art}, straightforwardly leverage the structure of an aggregated problem, and thus do not suffer from the identified issue. We provide a formal statement for this observation in Appendix \ref{sec:sat}.

\section{Limitation on the Bellman Equation with Universal Value Functions} \label{sec:Bell-universal}

    Provided a goal to achieve, a common approach in RL is to sample trajectories using some initial exploration policy, then improve the policy by reinforcing behaviors that lead to the goal. However, there is a common pitfall with this approach: the goal can be hard to achieve with the initial policy, leading to a lack of feedback to improve the policy.
    
    To address this challenge, Hindsight Experience Replay (HER) was introduced \citep{andrychowicz2017hindsight}. HER deviates from traditional RL by learning a \emph{universal value function} rather than focusing solely on the task-specific value function. A typical use of universal value functions is to estimate the reachability between states, predicting whether any particular state $b$ can be reached from any other state $a$. Then to reach a desired goal-state $g$, actions with high estimates of reaching $g$ are taken.

    However, despite HER's design to create feedback in challenging environments with sparse rewards, it can still overlook critical information in the problems. Moreover, under assumptions, the issue is independent of the prior used over universal value functions. Consequently, like classical value functions, improving the prior does not address the core limitation.
    
    To prove this limitation, we construct a counterexample similar to the one presented in the previous section. In this counterexample, a search guided by classical value functions learned using the Bellman equation struggles to find a solution. Furthermore, learning a universal value function in this scenario leverages no more feedback than learning a classical value function, leading to the same struggles in finding a solution. 

    \subsection{The Algorithm: Hindsight Experience Replay}
        As in the previous section, we work with CNF-SAT instances. We model the problem as taking a sequence of $n$ binary decisions with associated states, that incrementally build a candidate solution. To provide a clear goal-state for HER, we introduce a final step with two possible outcomes: $\mathrm{True}$ or $\mathrm{False}$. The state $\mathrm{True}$ indicates that the constructed candidate satisfies the problem, while $\mathrm{False}$ signifies failure. Thus, the target goal-state for HER is $g=\mathrm{True}$.
    
        The state space $S$ for a problem $p$ with $n$ variables is defined as $\{0,1\}^{[n]}\cup\{\mathrm{True},\mathrm{False}\}$, and the set of actions is binary: $\{0,1\}$.
    
        \begin{definition}\label{def:dyn-her}
            Given a CNF-SAT instance $p$ with $n$ variables, its \emph{dynamics $D^p$} is as follows:

            \begin{itemize}
                \item Start in the initial state $[]$.
                \item At each step (up to $n$), append the chosen binary action to the current state, forming a sequence of actions.
                \item For a state of length $n$ represented by the binary vector $x$, transition to $\mathrm{Check}(x;p)$.
            \end{itemize}
    
            The operator $D^p:S\times \{0,1\}\rightarrow S$ implements that dynamics for instance $p$, taking a state and a binary action as input and outputting the next state.
        \end{definition}
        
        We now define universal value functions estimating state-to-state reachability.
        \begin{definition}
            A \emph{universal value function} takes as input a CNF-SAT instance $p$, two states in $S$, and outputs a binary value. Moreover, it satisfies $v(s,s;p)=1$ for any state $s$.
        \end{definition}

        \begin{definition}
            A universal value function is \emph{optimal} if for any CNF-SAT instance $p$ with $n$ variables and any $s_1,s_2\in S$, $v(s_1,s_2;p)=1$ iff there exists a sequence of actions from $s_1$ that leads to $s_2$ under the instance's dynamics (Definition \ref{def:dyn-her}).
        \end{definition}

        As with classical value functions, the Bellman equation can be enforced to learn optimal universal value functions. For a universal value function $v$, an instance $p$, distinct states $a$ and $b$, and the dynamics operator $D^p$, the Bellman equation is:
        \begin{equation}\label{eq:be-uva}
            v(a,b;p)=\max\{v(D^p(a,0),b;p),\,v(D^p(a,1),b;p)\}.
        \end{equation}
        Algorithm \ref{alg:HER} implements this equation. It is initialized with a set of universal value functions representing a prior for SAT-solving. The algorithm iteratively samples candidate solutions with this prior while refining it with the Bellman equation.

        \begin{algorithm}[t]
        \caption{Learning universal value functions estimating state-to-state reachability with the Bellman equation (Hindsight Experience Replay) for CNF-SAT.\\
        Candidate solutions are iteratively sampled guided by a hypothesis set $V$ of universal value functions to reach a given goal-state $g$. Simultaneously, the Bellman equation is enforced on $V$ for pairs of sampled states and $g$. The sampling process is performed according to Definition \ref{def:pi-universal}.}
        \label{alg:HER}
        \textbf{Inputs:} $V^0:$ an initial set of value functions; $p:$ a CNF-SAT instance to solve with $n$ variables; $g:$ a final goal-state to reach
        \begin{algorithmic}
            \For{$t\gets 0,1,\ldots$}
                \State $s^t_0,s^t_1,\ldots,s^t_n,s^t_{n+1}\sim \pi^{V^t}(p,g)$
                \If{$s^t_{n+1}=g$}
                    \State \textbf{output} $s_n^t$
                \EndIf
                \State $V^{t+1}\gets V^t$
                \For{$i\gets 1,\ldots,n$}
                    \State $V^{t+1}\gets \{v\in V^{t+1}|\,v(s_i^t,g;p)=\max\{v(D^p(s_i^t,0),g;p),\,v(D^p(s_i^t,1),g;p)\}\}$
                    \For{$j\gets i+1,\ldots,n+1$}
                        \State $V^{t+1}\gets \{v\in V^{t+1}|\,v(s_i^t,s_j^t;p)=\max\{v(D^p(s_i^t,0),s_j^t;p),\,v(D^p(s_i^t,1),s_j^t;p)\}\}$
                    \EndFor
                \EndFor
            \EndFor
        \end{algorithmic}
        \end{algorithm}
    
        There are numerous pairs of states $a,b$ with which Equation \ref{eq:be-uva} can be enforced. Algorithm \ref{alg:HER} enforces the equation for pairs of states sampled in the same sequence and with the goal-state $g$. This is a key feature of HER that we replicate.

        To guide the construction of candidate solutions, a policy is derived from the universal value function hypotheses. This process follows a rule similar to Definition \ref{def:pi}.
        
        \begin{definition}\label{def:pi-universal}
            Given a set $V$ of universal value functions. The policy $\pi^V$ maps a CNF-SAT instance $p$ with $n$ variables and a goal $g$ to a distribution over sequences of states $s_0,\ldots,s_{n+1}$ in $S$.

            Define $V_{s,g,p}$ as the number of universal value functions in $V$ evaluating to $1$ for inputs $s$, $g$, and $p$: $\abs{\{v\in V|\,v(s,g;p)=1\}}$.

            The sequence of states is determined by the dynamics and the actions sampled at each state. The probability of taking action $0$ at some state $s$ is
            \begin{equation}\label{eq:pi-univ}
                \frac{V_{D^p(s,0),g,p}}{V_{D^p(s,0),g,p}+V_{D^p(s,1),g,p}},
            \end{equation}
            and the probability of action $1$ is the complementary. If the denominator is zero, both actions are taken with equal probabilities.
        \end{definition}

    \subsection{Counterexamples}
        We now define a monotonicity property for universal value functions, it parallels the one introduced earlier for classical value functions.
        \begin{definition}
            A universal value function $v$ is \emph{monotonic} if, for any CNF-SAT instance $p$ with $n$ variables, any state $s\in S$, any $x\in\{0,1\}^n$, and any $i,j\in[n]$ with $i<j$, $v$ satisfies $v(x_{\leq i},s;p)\geq v(x_{\leq j},s;p)$.
        \end{definition}
            
        \begin{restatable}{theorem}{herthm}\label{thm:HER}
            Let $p$ be a CNF-SAT instance with $n$ variables, constructed as an aggregation of CNF-SAT instances $p_1,\ldots,p_K$ with index lists $I_1,\ldots,I_k,\ldots,I_K$, where the first element of each $I_k$ is $k$. 
    
            Let $V$ and $V_1,\ldots,V_K$ represent sets of universal value functions and $g=\mathrm{True}$ denote the goal-state. 
            
            Assume the following:
            \begin{enumerate}
                \item The set $V$ factorizes into $V_1,\ldots,V_K$; that is, $v\in V$ iff there exists $v_1\in V_1,\ldots,v_K\in V_K$ such that, for all $x\in\{0,1\}^n$ and all $i\in[n]$, $v(x_{\leq i},g;p)=\prod\limits_{k\in[K]}v_{k}(x_{I_k\cap [i]},g;p)$.
    
                \item For an optimal value function $v^*$, and for all $k\in[K]$, either $v^*(0;p_k)=1$ or $v^*(1;p_k)=1$.
    
                \item For all $k\in[K]$ and $v\in V_k$, either $v(0,g;p_k)=1$ or $v(1,g;p_k)=1$.
                
                \item For an optimal value function $v^*$, and for all $k\in[K]$, if $v^*(0;p_k)=1$, then $\pi^{V_k}(s_0=0;p_k)\leq\pi^{V_k}(s_0=1;p_k)$; otherwise, $\pi^{V_k}(s_0=0;p_k)\geq\pi^{V_k}(s_0=1;p_k)$.
    
                \item Any $v\in V$ is monotonic.
    
                \item For any $v\in V$, $i,j\in[n]$ with $i\leq j$, and any $x_1\in\{0,1\}^i$, $x_2\in\{0,1\}^j$, $v(x_1,x_2;p)=v^*(x_1,x_2;p)$ where $v^*$ is an optimal universal value function.
    
                \item For an optimal universal value function $v^*$ and any state $s\in S$, if $v^*(s,\mathrm{False};p)=1$, then, for any $v\in V$, $v(s,\mathrm{False};p)=1$.
            \end{enumerate}
    
            Under these assumptions, Algorithm \ref{alg:HER} initialized with $V^0=V$, $p=p$, and $g=g$, runs for an expected time of at least $2^{K-1}$ steps.
        \end{restatable}
    
        The proof is similar to the one of Theorem \ref{thm:BE} and is provided in Appendix \ref{sec:proofs}.
    
        Theorem \ref{thm:HER} relies on similar assumptions to Theorem \ref{thm:BE}. Assumptions (1), (3), (4), and (5) are adapted to universal value functions, while assumption (2) remains unchanged. Assumptions (6) and (7) are newly introduced.
        
        Assumption (6) ensures that every universal value function in the initial set $V^0$ accurately predicts which partial candidate solutions can be reached from given partial candidate solutions. This condition is satisfied by any optimal value function and is independent of the specific problem being solved.
        
        Assumption (7) requires that any value function in $V^0$ can determine whether a partial candidate must lead to a full solution, regardless of subsequent actions. This condition is also satisfied by any optimal universal value function.  
    
        Theorem \ref{thm:HER} demonstrates that an application of HER with universal value functions for state-to-state reachability does not circumvent the issue identified in the previous section regarding the Bellman equation with classical value functions.
        
\section{Discussion} \label{sec:disc}
    \paragraph{Summary}
    In this paper, we defined two algorithms for problem-solving with learning-based guidance in their search processes, Algorithms \ref{alg:BE_search} and \ref{alg:HER}. Both of these algorithms represent an approach of the RL literature, they enforce the Bellman equation on sampled states, one to learn a value function, the other a universal value function for state-to-state reachability (Hindsight Experience Replay).
    
    To analyze these algorithms, we constructed counterexample problems that, by design, have a clear structure to exploit, yet remain provably challenging for the respective algorithms, Theorems \ref{thm:BE} and \ref{thm:HER}. Our proofs demonstrate that these algorithms struggle because they do not leverage rich feedback from their failed attempts to improve their guidance.

    \paragraph{Limitations}
    One limitation of our theoretical analysis is its reliance on counterexamples, which, by nature, are specific problems. This raises the question: do these counterexamples reflect broader limitations on the studied algorithmic approaches with practical problems?

    We argue that they do. Our counterexamples and analysis are based on aggregating sub-problems into a larger problem, an expected structure in practical applications. Also, although we use SAT instances to build minimal clean counterexamples, our analysis is not tied to this problem, and we believe that the results generalize to other domains, such as Automated Theorem Proving.
    
    A second limitation of our theoretical analysis is its reliance on specific algorithmic implementations. While our algorithms are straightforward implementations of each approach, they may inadvertently introduce specific implementation choices not part of the original approaches and limit the scope of our findings. 
     
    For example, our algorithms use Bayesian Learning by iteratively updating a set of hypotheses according to the sampled data. While Bayesian Learning is an idealized framework to describe a priori information and learning, it is not a common choice in the deep RL literature.
     
    Another example of a specific choice in our algorithms is the use of universal value functions to model state-to-state reachability. In their original definitions, universal value functions and HER are more flexible, they can be used to model the reachability of some abstract goals not just future states. This additional flexibility could improve feedback and allow to avoid the studied limitation for algorithms developed in that direction.

    \paragraph{Conclusion}
    Despite the use of specific algorithms and counterexamples, we believe our findings highlight a broad practical limitation on algorithms that rely on the Bellman equation enforced at sampled states. These include model-free Deep RL methods like Deep Q-learning \citep{mnih2013playing} and actor-critic methods \citep{schulman2017proximal}, as well as some model-based methods like AlphaZero \citep{silver2018general}, all of which depend on the Bellman equation to learn to guide the search.

    Our work formalizes inefficiencies in these classical algorithmic ideas of RL. These insights can inform the development of novel algorithms better leveraging learning to accelerate search, leading to more effective problem-solving techniques.

\section*{Acknowledgements}
RJ is a FNRS honorary Research Associate. This project has received funding from the European Research Council (ERC) under the European Union's Horizon 2020 research and innovation programme under grant agreement No 864017 - L2C, from the Horizon Europe programme under grant agreement No101177842 - Unimaas, and from the ARC (French Community of Belgium)- project name: SIDDARTA.

\bibliography{biblio}

\begin{thebibliography}{15}
\providecommand{\natexlab}[1]{#1}
\providecommand{\url}[1]{\texttt{#1}}
\expandafter\ifx\csname urlstyle\endcsname\relax
  \providecommand{\doi}[1]{doi: #1}\else
  \providecommand{\doi}{doi: \begingroup \urlstyle{rm}\Url}\fi

\bibitem[Andrychowicz et~al.(2017)Andrychowicz, Wolski, Ray, Schneider, Fong,
  Welinder, McGrew, Tobin, Pieter~Abbeel, and
  Zaremba]{andrychowicz2017hindsight}
Marcin Andrychowicz, Filip Wolski, Alex Ray, Jonas Schneider, Rachel Fong,
  Peter Welinder, Bob McGrew, Josh Tobin, OpenAI Pieter~Abbeel, and Wojciech
  Zaremba.
\newblock Hindsight experience replay.
\newblock \emph{Advances in neural information processing systems}, 30, 2017.

\bibitem[Ayg{\"u}n et~al.(2022)Ayg{\"u}n, Anand, Orseau, Glorot, Mcaleer,
  Firoiu, Zhang, Precup, and Mourad]{aygun2022proving}
Eser Ayg{\"u}n, Ankit Anand, Laurent Orseau, Xavier Glorot, Stephen~M Mcaleer,
  Vlad Firoiu, Lei~M Zhang, Doina Precup, and Shibl Mourad.
\newblock Proving theorems using incremental learning and hindsight experience
  replay.
\newblock In \emph{International Conference on Machine Learning}, pp.\
  1198--1210. PMLR, 2022.

\bibitem[Biere et~al.(2009)Biere, Heule, and van Maaren]{biere2009handbook}
Armin Biere, Marijn Heule, and Hans van Maaren.
\newblock \emph{Handbook of satisfiability}, volume 185.
\newblock IOS press, 2009.

\bibitem[Knuth(2015)]{knuth2015art}
Donald~E Knuth.
\newblock \emph{The art of computer programming, Volume 4, Fascicle 6:
  Satisfiability}.
\newblock Addison-Wesley Professional, 2015.

\bibitem[Mnih et~al.(2013)Mnih, Kavukcuoglu, Silver, Graves, Antonoglou,
  Wierstra, and Riedmiller]{mnih2013playing}
Volodymyr Mnih, Koray Kavukcuoglu, David Silver, Alex Graves, Ioannis
  Antonoglou, Daan Wierstra, and Martin Riedmiller.
\newblock Playing atari with deep reinforcement learning.
\newblock \emph{arXiv preprint arXiv:1312.5602}, 2013.

\bibitem[Pinon et~al.(in press)Pinon, Jungers, and
  Delvenne]{pinon2025limitation}
Brieuc Pinon, Rapha{\"e}l Jungers, and Jean-Charles Delvenne.
\newblock A limitation on black-box dynamics approaches to reinforcement
  learning.
\newblock \emph{Transactions on Machine Learning Research}, in press.

\bibitem[Poesia et~al.(2024)Poesia, Broman, Haber, and
  Goodman]{poesia2024learning}
Gabriel Poesia, David Broman, Nick Haber, and Noah~D Goodman.
\newblock Learning formal mathematics from intrinsic motivation.
\newblock \emph{arXiv preprint arXiv:2407.00695}, 2024.

\bibitem[Schaul et~al.(2015)Schaul, Horgan, Gregor, and
  Silver]{schaul2015universal}
Tom Schaul, Daniel Horgan, Karol Gregor, and David Silver.
\newblock Universal value function approximators.
\newblock In \emph{International conference on machine learning}, pp.\
  1312--1320. PMLR, 2015.

\bibitem[Schulman et~al.(2017)Schulman, Wolski, Dhariwal, Radford, and
  Klimov]{schulman2017proximal}
John Schulman, Filip Wolski, Prafulla Dhariwal, Alec Radford, and Oleg Klimov.
\newblock Proximal policy optimization algorithms.
\newblock \emph{arXiv preprint arXiv:1707.06347}, 2017.

\bibitem[Silva \& Sakallah(1996)Silva and Sakallah]{silva1996grasp}
JP~Marques Silva and Karem~A Sakallah.
\newblock Grasp-a new search algorithm for satisfiability.
\newblock In \emph{Proceedings of International Conference on Computer Aided
  Design}, pp.\  220--227. IEEE, 1996.

\bibitem[Silver et~al.(2018)Silver, Hubert, Schrittwieser, Antonoglou, Lai,
  Guez, Lanctot, Sifre, Kumaran, Graepel, et~al.]{silver2018general}
David Silver, Thomas Hubert, Julian Schrittwieser, Ioannis Antonoglou, Matthew
  Lai, Arthur Guez, Marc Lanctot, Laurent Sifre, Dharshan Kumaran, Thore
  Graepel, et~al.
\newblock A general reinforcement learning algorithm that masters chess, shogi,
  and go through self-play.
\newblock \emph{Science}, 362\penalty0 (6419):\penalty0 1140--1144, 2018.

\bibitem[Sun et~al.(2019)Sun, Jiang, Krishnamurthy, Agarwal, and
  Langford]{sun2019model}
Wen Sun, Nan Jiang, Akshay Krishnamurthy, Alekh Agarwal, and John Langford.
\newblock Model-based rl in contextual decision processes: Pac bounds and
  exponential improvements over model-free approaches.
\newblock In \emph{Conference on learning theory}, pp.\  2898--2933. PMLR,
  2019.

\bibitem[Sutton \& Barto(2018)Sutton and Barto]{sutton2018reinforcement}
Richard~S Sutton and Andrew~G Barto.
\newblock \emph{Reinforcement learning: An introduction}.
\newblock MIT press, 2018.

\bibitem[Sutton et~al.(2011)Sutton, Modayil, Delp, Degris, Pilarski, White, and
  Precup]{sutton2011horde}
Richard~S Sutton, Joseph Modayil, Michael Delp, Thomas Degris, Patrick~M
  Pilarski, Adam White, and Doina Precup.
\newblock Horde: A scalable real-time architecture for learning knowledge from
  unsupervised sensorimotor interaction.
\newblock In \emph{The 10th International Conference on Autonomous Agents and
  Multiagent Systems-Volume 2}, pp.\  761--768, 2011.

\bibitem[Trinh et~al.(2024)Trinh, Wu, Le, He, and Luong]{trinh2024solving}
Trieu~H Trinh, Yuhuai Wu, Quoc~V Le, He~He, and Thang Luong.
\newblock Solving olympiad geometry without human demonstrations.
\newblock \emph{Nature}, 625\penalty0 (7995):\penalty0 476--482, 2024.

\end{thebibliography}
\bibliographystyle{bib_style}

\clearpage
\newpage
\appendix

\section{Proofs of the Main Theorems}\label{sec:proofs}

    \bellmanvaluethm*
    \begin{proof}
        By condition (2) each first variable of all instances $p_1,\ldots,p_K$ must be either $0$ or $1$ to be a solution. Thus, any solution to instance $p$ must have its first $K$ variables equal to some unique binary vector $y^*\in\{0,1\}^K$.

        We derive an upper bound on the probability of generating this initial vector, then conclude from it the lower bound on the running time stated in the Theorem.

        We prove by induction on $t$ (the number of loops performed) that the probability of generating $y^*$ is the smallest inside some set of $2^K-t$ binary sequences of length $K$ as long as $y^*$ has not been generated.

        For $t=0$, by (1) and (4), $\pi^{V^0}$ assigns the lowest probability on $y^*$ over the $2^K$ possibilities.

        Assuming the inductive hypothesis holds at some loop $t$ for some set $S$ of sequence with size $2^K-t$ and the sequence $x^t_{\leq K}$ is not $y^*$. We look at the effect of the Bellman equations enforced on the set of value functions with two cases: $i<K$ and $i\geq K$ with the additional consistency equation for complete assignments.
        
        If a Bellman equation is applied for $i<K$ then we show that no value functions are removed. Two possibilities: first, if $v(x_{\leq i}^t;p)=0$ then, by the monotonicity assumption (5), $v([x_{\leq i}^t,0];p)=0$ and $v([x_{\leq i}^t,1];p)=0$; second, if $v(x_{\leq i}^t;p)=1$ then, by (1) and (3), at least $v([x_{\leq i}^t,0];p)$ or $v([x_{\leq i}^t,1];p)$ must evaluate to $1$. The equation is thus satisfied in both possibilities.

        Now, in the second case, the Bellman equation is applied with $i\geq K$, including the condition ensuring consistency with $\mathrm{Check}$. Let $v\in V^t$ be a value function that is removed by the Bellman equation for some $i\geq K$. We know $v(x^t_{\leq K};p)=1$ by the monotonicity assumption (5), otherwise the value function evaluates to $0$ in all possible continuations and the Bellman equation is satisfied. Knowing this, by assumptions (1) and (3), the value function outputs $1$ for $x^t_{\leq K}$ and $0$ for all the other sequences of length $K$. Consequently, from Definition \ref{def:pi}, at most one element of $S$ decreases in probability, and all the others increase by a common normalizing factor. This positive factor preserves the order between the sequences' probabilities of being generated among this set, and thus the inductive hypothesis holds for $t\gets t+1$.

        This bound leads to an expected number of loops of at least $2^{K-1}$ to generate $y^*$ and thus a solution.
    \end{proof}

    \herthm*
    \begin{proof}

        We reduce our claim to that of Theorem \ref{thm:BE}. The distribution used to sample candidate solutions follows the same constraints as in Theorem \ref{thm:BE}, with the value functions that guide the search, $v(.;p)$, being replaced by universal value functions evaluated with the goal $g$: $v(.,g;p)$ for $v\in V$. The set $V^t$, when applied with $g$, follows the same constraints as in the previous theorem, with the exception that the Bellman equation is not only enforced with the final state $g$ but also with the state $\mathrm{False}$ and states corresponding to partial solutions (innermost statement in the loops of Algorithm \ref{alg:HER}).
        
        We study the situation where the Bellman equation is applied with the state $s_{n+1}^t=$False as a second argument. The term $v(s_i^t,s_{n+1}^t;p)$ is either $0$ or $1$. If it evaluates to $0$, by assumption (7), False is not reachable from $s_i^t$ and thus $s_{n+1}^t=\mathrm{False}$ is not possible, we have a contradiction, this case is not possible. For the case $v(s_i^t,s_{n+1}^t;p)=1$, we remark that at least one of $D^p(s_i^t,0)$ or $D^p(s_i^t,1)$ leads to the state False (since it was reached in this trajectory), thus, by assumption 7, one of $v(D^p(s_i^t,0),s_{n+1}^t;p)$ or $v(D^p(s_i^t,1),s_{n+1}^t)$ evaluates to $1$. Consequently, all the universal value functions in $V^t$ satisfy the Bellman equation with $s_{n+1}^t/\mathrm{False}$ as a second argument, and none are removed.

        The Bellman equation procedure is also called with the second state corresponding to a partial solution. By assumption (6), in that case, any universal value function in $V^0$ (and thus $V^t$) already matches the output of an optimal universal value function and is consistent with the Bellman equation. Again, none of the universal value functions are removed by enforcing the equation.

        Therefore, under the assumptions, the universal value functions in set $V^t$ evaluated with $g=\mathrm{True}$ in Algorithm \ref{alg:HER}, are updated identically to the set of value functions in Algorithm \ref{alg:BE_search}. The Theorem follows.
    \end{proof}

\section{Positive Result for a Resolution-Based SAT Solver}\label{sec:sat}
    Our main text presented a design to produce counterexamples for algorithms based on the Bellman equation approach with classical value functions and universal value functions (HER). In this section, we prove that our design does not imply a limitation for SAT solvers based on the resolution operator. For these algorithms, the aggregation structure designed in our counterexamples does not incur an intractable computational cost. Enforcing the idea that being unable to decompose this structure is a limitation. 
    
    We note that this result applies to an algorithm constrained by a fixed variable assignment ordering, thus under the same context as the negative results for Bellman-based RL.

    We first define a resolution-based SAT solver, and then provide the positive result.

    Our SAT solver is a minimalistic version of modern conflict-driven clause learning (CDCL) SAT solvers \citep{silva1996grasp,biere2009handbook,knuth2015art}. At their core is a procedure that iteratively tries to generate a solution, and, simultaneously to these attempts, learns from its failures. This working principle is thus very similar to RL algorithms, the main difference is found where learning happens. Here, learning consists in deducting new logical formulas with the resolution operator.

    We define this operator here, it allows us to learn a new clause by deduction.
    \begin{definition}
        For some CNF-SAT instance, let $x_i$ be some variable and $l_1,\ldots,l_m,l_1',\ldots l_{m'}'$ be some literals. Given two clauses of the form $x_i\lor l_1\lor \ldots \lor l_m$ and $\lnot x_i\lor l_1'\lor \ldots,l_{m'}'$, the \emph{resolution operator} produces the new clause $l_1\lor\ldots\lor l_m\lor l_1'\lor \ldots \lor l_{m'}'$.
    \end{definition}

    To describe the algorithm, we need two other concepts. As the algorithm incrementally builds a solution, some variables will be assigned and others not. We say that a variable $x_i$ is \emph{forced} to $1(/0)$ if a clause as the form $(\lnot)x_i\lor l_1,\ldots\lor l_m$ where $l_1,\ldots, l_m$ are all literals that evaluate to False due to already assigned variables. A \emph{conflict} is when the variable is forced to $0$ and $1$ by different clauses, in which case the current assignments cannot lead to a solution and these clauses can be resolved.

    We now define Algorithm \ref{alg:resolution} iteratively performing generation loops where a candidate solution is incrementally generated.

    \begin{algorithm}[t]
    \caption{Resolution-based SAT solver.\\
    Iteratively attempt to generate a solution, each attempt incrementally assigns the variables until a complete solution is found or there is a conflict. In case of a conflict, the resolution operator is applied and the new clause is learned to avoid it in future attempts.}
    \label{alg:resolution}
    \textbf{Inputs:} $C^0:$ the set of clauses of a CNF-SAT instance with $n$ variables
    \begin{algorithmic}
        \For{$t\gets 1,2,\ldots$}
            \For{$i\gets 1,\ldots,n$}
                \If{contradiction on $x_i$ with clauses in $C^{t-1}$} 
                    \State $C^t\gets C^{t-1}\cup\{\mathrm{resolution}(c_0,c_1)|\,c_0,c_1\in C^{t-1},\, c_0,c_1\text{ in conflict}\}$
                    \State break
                \ElsIf{$x_i$ is forced to $1$}
                    \State $x_i\gets 1$
                \ElsIf{$x_i$ is forced to $0$}
                    \State $x_i\gets 0$
                \Else
                    \State $x_i\gets 0$
                \EndIf
            \EndFor
            \If{$x$ is completely assigned}
                \State \textbf{output} $x$
            \Else
                \State reinitialize assignments
            \EndIf
        \EndFor
    \end{algorithmic}
    \end{algorithm}

    \begin{theorem}\label{thm:resolution}
        Let $p$ be an aggregation (Definition \ref{def:aggregation}) of CNF-SAT instances $p_1,\ldots,p_K$ with index lists $I_1,\ldots,I_K$. Let $p_1',\ldots,p_K'$ be the same CNF-SAT instances with their variables' indices permuted by their respective order in $I_1,\ldots,I_K$, and let $T_1,\ldots,T_K$ be the number of failed attempt loops of Algorithm \ref{alg:resolution} on those instances.

        Then Algorithm \ref{alg:resolution} with input $p$ performs at most $\sum_{k=1}^K T_k$ failed attempt loops.
    \end{theorem}
    \begin{proof}
        We show that the problems are concurrently and independently solved by Algorithm \ref{alg:resolution}. Each iteration of the algorithm on $p$ simulates the iteration of the algorithm on one of the sub-problems.
        
        The state of the algorithm evolving across loops is described by the set of clauses $C^t$. Let $C^{t_k}_k$ be the set of clauses of the algorithm on problem $p_k'$ at some loop $t_k\in\N$. Inductively, at any loop $t\in\N$, $C^{t-1}$ decomposes into the set of clauses $C^{t_1-1}_1,\ldots, C^{t_K-1}_K$ for some $t_1,\ldots,t_K\in\N$ up to a variable mapping (corresponding to the inverse of the index lists, $I_k$, composed with the permutation from $p_k$ to $p_k'$).

        This is true at initialization $t=1$ with $t_1,\ldots,t_K=1$.

        If the inductive hypothesis is true at some loop $t$ for some $t_1,\ldots,t_k$, then the generation process assigns identically the variables as the algorithm applied on the individual sub-problems at respective loops $t_1,\ldots,t_k$ and has a conflict at the same variables with the same clauses up to the variable mapping. Let $x_i$ with $i\in I_k$ be the first variable leading to a conflict in the assignment process. Up to the variable mapping, the same clauses lead to a conflict as in problem $p_k'$, thus the resolution operator computes the same resolved clauses that are added to $C_k^{t_k}$ at iteration $t_k$. The induction hypothesis is proved for the loop $t+1$.
        
        Moreover, at each attempt loop, one of the $t_k$ is incremented while the others stay the same. After, $\sum_{k\in[K]}T_k$ loops the assignment process is successful on all the sub-problems, and a solution is found.

    \end{proof}

    Theorem \ref{thm:resolution} states that for Algorithm \ref{alg:resolution} solving the aggregation of a set of problems or solving each problem in the set independently takes similar time (up to a permutation of the variables' indices). The proof leverages the fact that the resolution operator keeps the sub-problems independent.

\end{document}